\documentclass[orivec]{llncs}

\newif\ifdraft
\draftfalse

\usepackage{bbding}
\usepackage{latexsym,amssymb,amsmath}
\usepackage{xspace}
\usepackage{comment}
\usepackage{url}
\usepackage{enumerate}
\usepackage[defblank]{paralist}
\usepackage{graphicx}
\graphicspath{{imgs/}}
\usepackage{subfigure}
\usepackage{microtype}
\usepackage{tabularx}
\usepackage{hyperref}
\usepackage{times}

\usepackage[normalem]{ulem} 
\ifdraft
  \usepackage{mychanges}
  \usepackage[colorinlistoftodos]{todonotes}
\else
  \usepackage[final]{mychanges}
  \usepackage[disable]{todonotes}
\fi
\setauthormarkup[right]{{\scriptsize[#1]}}
\definechangesauthor{DC}{blue}
\definechangesauthor{AS}{orange}
\definechangesauthor{MM}{green}
\definechangesauthor{DS}{red}

\newcommand{\A}{\mathcal{A}} 
\newcommand{\C}{\mathcal{C}} 
 \newcommand{\F}{\mathcal{F}}
 
\newcommand{\I}{\mathcal{I}} 
 \renewcommand{\L}{\mathcal{L}}
\newcommand{\M}{\mathcal{M}} 
\renewcommand{\O}{\mathcal{O}} 
 \newcommand{\R}{\mathcal{R}}
\renewcommand{\S}{\mathcal{S}} \newcommand{\T}{\mathcal{T}}
 \newcommand{\V}{\mathcal{V}}

\newcommand{\mf}{\mathfrak}

\newcommand{\per}{\mbox{\bf .}}                  

\newcommand{\set}[1]{\{#1\}}                      

\newcommand{\tup}[1]{\langle#1\rangle}            

\newcommand{\dom}[1][\I]{\Delta^{#1}}  
\newcommand{\Int}[2][\I]{#2^{#1}}      

\newcommand{\SOME}[2]{\exists #1 \per #2}
\newcommand{\SOMET}[1]{\exists #1}

\newcommand{\INV}[1]{#1^{-}}

\newcommand{\vfo}{\ensuremath{v}}
\newcommand{\vso}{\ensuremath{V}}

\newcommand{\DIAM}[1]{\langle - \rangle #1}

\newcommand{\ISA}{\sqsubseteq}

\newcommand{\limp}{\rightarrow}

\newcommand{\MODA}[1]{(#1)_{\vfo,\vso}^\mf{A}}   

\newcommand{\MODAX}[2]{(#1)_{\vfo #2,\vso}^\mf{A}}   

\newcommand{\MODAZ}[2]{(#1)_{\vfo,\vso #2}^\mf{A}}   

\newcommand{\true}{\mathsf{true}}
\newcommand{\false}{\mathsf{false}}

 \newcommand{\dllite}{\textit{DL-Lite}\xspace}
 
 \newcommand{\dlliter}{\textit{DL-Lite}\ensuremath{_{\mathcal{R}}}\xspace}

\newcommand{\ans}[2][]{\mathit{ans}_{#1}(#2)}
\newcommand{\rew}[1]{\mathit{rew}(#1)}

\newcommand{\conj}{\mathit{conj}}

\newcommand{\map}[2]{#1 \rightsquigarrow #2}

\newcommand{\adom}[1]{\textsc{adom}(#1)}

\newcommand{\Ans}[2][]{\textsc{ans}_{#1}(#2)}

\newcommand{\abox}{\mathit{abox}}

\newcommand{\CONST}{\C}

\newcommand{\schema}{\ensuremath{\R}\xspace}
\newcommand{\DBschema}{\schema}
\newcommand{\idb}{\ensuremath{\I_0}\xspace}

\newcommand{\sys}{\ensuremath{\S}\xspace}

\newcommand{\sgds}{\mbox{SAS}\xspace}
\newcommand{\sgdsSym}{\ensuremath{\S}\xspace}

\newcommand{\DBinst}{\I}
\newcommand{\DBinit}{\idb} 
\newcommand{\VALUE}{\V}
\newcommand{\Instances}{\Gamma}

\newcommand{\mapping}{\M}

\newcommand{\obdaSym}{\ensuremath{\O}\xspace}
\newcommand{\obdaTup}{\tup{\DBschema, \TBox, \mapping}}

\newcommand{\unfold}[2]{{\textsc{unfold}(#1,#2)}}

\newcommand{\qunsat}[1]{\ensuremath{\mathsf{q}_\mathsf{unsat}(#1)}}

\newcommand{\skolem}{\textsc{fs}}

\newcommand{\TSAs}{\Sigma}
\newcommand{\TSAt}{\Rightarrow}
\newcommand{\TSAsym}{\Upsilon^{\scriptsize{\text{R}}}} 
\newcommand{\TSAtup}{\tup{\R, \TSAs, s_0, \db, \TSAt}}
\newcommand{\TBox}{\T}

\newcommand{\TSSs}{\Sigma}
\newcommand{\TSSt}{\Rightarrow}
\newcommand{\TSSsym}{\Upsilon^{\scriptsize{\text{S}}}} 
\newcommand{\TSStup}{\tup{\TBox, \TSSs, s_0, \abox, \TSSt}}

\newcommand{\muladom}{\ensuremath{\muL_{A}}\xspace}

\newcommand{\ctla}{ \textnormal{CTL}\ensuremath{_A}\xspace}

\newcommand{\muladl}{\ensuremath{\muladom^{{\textnormal{EQL}}}}\xspace}

\newcommand{\ctladl}{\ensuremath{\ctla^{{\textnormal{EQL}}}}\xspace}

\newcommand{\inadom}{\ensuremath{\textsc{live}}}

\newcommand{\muL}{\mu\L} 

\renewcommand{\mf}[1]{\Upsilon}
\newcommand{\db}{\mathit{db}}

\newcolumntype{R}{>{\raggedleft\arraybackslash}X}
\newcolumntype{L}{>{\raggedright\arraybackslash}X}
\newcolumntype{C}{>{\centering\arraybackslash}X}

\renewcommand{\A}{A}
\renewcommand{\T}{T}
\renewcommand{\R}{R}

\newcommand{\disjoint}[2]{\ensuremath{\mathsf{disjoint}(#1,#2)}}

\newcommand{\val}[1]{\mathit{val}}

\newcommand{\evaluate}{\mathit{eval}}

\newcommand{\ontop}{\textsc{-ontop-}\xspace}
\newcommand{\systemName}{OBGSM\xspace}
\newcommand{\marianoSystem}{\ontop}

\newcommand{\Jena}{Apache Jena$^{\textsc{TM}}$\xspace}

\newcommand{\AG}{\textsc{AG}}
\newcommand{\EG}{\textsc{EG}}
\newcommand{\AF}{\textsc{AF}}
\newcommand{\EF}{\textsc{EF}}
\newcommand{\AX}{\textsc{AX}}
\newcommand{\EX}{\textsc{EX}}

\newcommand{\Auntil}[2]{\textsc{A}~(#1~\textsc{until}~#2)}
\newcommand{\Euntil}[2]{\textsc{E}~(#1~\textsc{until}~#2)}

\newcommand{\andTemp}[2]{#1~\textsc{and}~#2}
\newcommand{\orTemp}[2]{#1~\textsc{or}~#2}
\newcommand{\orTempB}{~\textsc{or}~}
\newcommand{\notTemp}{\textsc{!}}
\newcommand{\implTemp}[2]{#1~\texttt{->}~#2}
\newcommand{\implTempB}{~\texttt{->}~}

\newcommand{\existsTemp}{\textsc{exists}\xspace}
\newcommand{\forallTemp}{\textsc{forall}\xspace}

\newcommand{\existsQuantification}{{\tt eQuantification}\xspace}
\newcommand{\forallQuantification}{{\tt fQuantification}\xspace}

\newcommand{\formula}{\Phi}
\newcommand{\query}{Q}

\newcommand{\variables}{{\textbf{Var}}\xspace}

\newcommand{\sparql}{SPARQL 1.1\xspace}

\newcommand{\sparqlSELECT}{\textsc{Select}\xspace}

\newcommand{\sparqlWHERE}{\textsc{Where}\xspace}

\newcommand{\expression}{\psi}

\newcommand{\notGSM}{{\tt !}}
\newcommand{\andGSM}{\textsc{and}}
\newcommand{\orGSM}{\textsc{or}}

\newcommand{\getGSM}{\textbf{get}}
\newcommand{\forallGSM}{\textbf{forall}}
\newcommand{\existsGSM}{\textbf{exists}}

\newcommand{\variableGSM}{variable}
\newcommand{\constantGSM}{{\tt constant}}

\newcommand{\aopGSM}{\textbf{aop}}
\newcommand{\lopGSM}{\textbf{lop}}

\newcommand{\GSMisMSachieved}{\mathit{GSM.isMilestoneAchieved}}

\begin{document}

\title{Verification of Semantically-Enhanced Artifact Systems\thanks{This
  research has been partially supported by the EU under the ICT Collaborative
  Project ACSI (Artifact-Centric Service Interoperation), grant agreement
  n.~FP7-257593.}}

\subtitle{Extended Version}

\author{Babak Bagheri Hariri \and Diego Calvanese \and
 Marco Montali \and \\
 Ario Santoso \and
 Dmitry Solomakhin
}

\authorrunning{B.~Bagheri~Hariri, D.~Calvanese, M.~Montali, A.~Santoso,
 D.~Solomakhin}

\institute{KRDB Research Centre for Knowledge and Data, Free University of Bozen-Bolzano\\
 \email{\textit{lastname}@inf.unibz.it} }

\maketitle

\begin{abstract} 

  Artifact-Centric systems have emerged in the last years as a suitable framework to model business-relevant entities, by combining their static and dynamic aspects. In particular, the Guard-Stage-Milestone (GSM) approach has been recently proposed to model artifacts and their lifecycle in a declarative way. In this paper, we enhance GSM with a Semantic Layer, constituted by a full-fledged OWL 2 QL ontology linked to the artifact information models through mapping specifications. The ontology provides a conceptual view of the domain under study, and allows one to understand the evolution of the artifact system at a higher level of abstraction.  In this setting, we present a technique to specify temporal properties expressed over the Semantic Layer, and verify them according to the evolution in the underlying GSM model. This technique has been implemented in a tool that exploits state-of-the-art ontology-based data access technologies to manipulate the temporal properties according to the ontology and the mappings, and that relies on the GSMC model checker for verification.

\end{abstract}

\section{Introduction}
\label{sec:introduction}

In the last decade, the marriage between processes and data has been
increasingly advocated as a key objective towards a comprehensive modeling and
management of complex enterprises \cite{CaDM13}. This requires to go beyond
classical (business) process specification languages, which largely leave the
connection between the process dimension and the data dimension underspecified,
and to consider data and processes as ``two sides of the same coin''
\cite{Reic12}.  In this respect, artifact-centric systems \cite{NiCa03,Hull08}
have lately emerged as an effective framework to model business-relevant
entities, by combining in a holistic way their static and dynamic
aspects. Artifacts are characterized by an "information model", which maintains
the artifact data, and by a lifecycle that specifies the allowed ways to
progress the information model.  Among the different proposals for
artifact-centric process modelling, the Guard-Stage-Milestone (GSM) approach
has been recently proposed to model artifacts and their lifecycle in a
declarative, flexible way \cite{HDDF*11}. GSM is equipped with a formal
execution semantics \cite{DaHV13}, which unambiguously characterizes the
artifact progression in response to external events. Notably, several key
constructs of the emerging OMG standard on Case Management and Model Notation
\footnote{\url{http://www.omg.org/spec/CMMN/}} have been borrowed from GSM.

Despite the tight integration between the data and process component, the
artifact information model typically relies on relatively simple structures,
such as (nested) lists of key-value pairs.  This causes an abstraction gap
between the high-level, conceptual view that business stakeholders have of
domain-relevant entities and relations, and the low-level representation
adopted inside artifacts.  To overcome this problem, in \cite{CDLMS12} it is
proposed to enhance artifact systems with a Semantic Layer, constituted by a
full-fledged ontology linked to the artifact information models through mapping
specifications. On the one hand, the ontology provides a conceptual view of the
domain under study, and allows one to understand the evolution of the artifact
system at a higher level of abstraction. On the other hand, mapping
specifications allow one to connect the elements present in the Semantic Layer
with the concrete data maintained in the artifact information models, relying
on an approach recently proposed in the context of Ontology-Based Data Access
(OBDA) \cite{CDLL*09}.

To specify the ontology, we adopt the OWL~2~QL profile \cite{MCHW*12} of the
standard Web Ontology Language (OWL) \cite{BaEA12}, since it supports efficient
access to data while fully taking into account the constraints expressed in the
ontology.  Specifically, OWL~2~QL enjoys so-called \emph{first-order
 rewritability} of query answering \cite{CDLLR07}, which guarantees that
conjunctive queries posed over the ontology can be rewritten into first-order
queries that incorporate the ontological constraints, and thus do not require
further inference for query answering.  A rewritten query can then be
\emph{unfolded} with respect to the mappings, so as to obtain a query
formulated over the underlying data that can be directly evaluated so as to
obtain the so-called \emph{certain answers} to the original query over the
ontology \cite{PLCD*08}.

We follow here an approach that is similar to the one proposed in
\cite{CDLMS12}, and that is based on specifying in terms of the Semantic Layer
dynamic/temporal laws that the system should obey, and that need to be verified
according to the evolution in the underlying artifact layer.  However,
differently from \cite{CDLMS12}, in which the Semantic Layer is mainly used to
\emph{govern} the progression of artifacts, by forbidding the execution of
actions that would lead to violate the constraints in the ontology, here we are
primarily interested in exploiting the Semantic Layer to ease the specification
of the dynamic/temporal laws.
In this light, we extend the technique provided in \cite{CDLMS12}
by relying on a more expressive verification formalism, which supports
first-order epistemic queries embedded into an expressive temporal language,
the first-order $\mu$-calculus \cite{Emer96b}, while allowing for
quantification across states.  The latter makes it possible to predicate over
the temporal evolution of individuals, and thus represents an enhancement that
is fundamental for capturing many practical scenarios (cf.\
Section~\ref{sec:caseStudy}).

We then concretize this formal framework by relying on GSM as the artifact
model, and by exploiting state of the art technologies, on the one hand for
dealing with the ontology and the mappings, and on the other hand for
performing verification.  Specifically, we have developed the OBGSM tool (see
Section~\ref{sec:obgsmtool}), which is able to reformulate temporal properties
expressed over the ontology in terms of the underlying GSM information model,
and which adopts the state of the art OBDA system Quest \cite{RoCa12b} to
efficiently rewrite and unfold the epistemic queries embedded in the temporal
property to verify.  For the actual verification task, we rely on the recently
developed GSMC model checker for GSM \cite{BeLP12b}, which is the only model
checker currently available that is able to verify temporal formula over
artifact systems.  GSMC adopts as temporal formalism a variant of the
first-order branching time logic CTL \cite{ClGP99} with a restricted form of
quantification across states\footnote{Note that CTL can be expressed in the
 alternation-free fragment of the $\mu$-calculus \cite{Dam92,Emer96b}.}.  The
restrictions imposed by GSMC made it necessary to suitably accommodate the
mapping language and the temporal formalism over the Semantic Layer so as to
ensure that the temporal formulas resulting from rewriting and unfolding can be
effectively processed by GSMC.


\section{Preliminaries}
\label{sec:preliminaries}

OWL~2~QL is a profile\footnote{In W3C terminology, a \emph{profile} is a
 sublanguage defined by suitable syntactic restrictions.} of the Web Ontology
Language OWL~2 standardized by the W3C.
OWL~2~QL is specifically designed for building an ontology layer to
wrap possibly very large data sources.
Technically, OWL~2~QL is based on the description logic \dlliter, which is a
member of the \dllite family \cite{CDLLR07}, designed specifically for
effective ontology-based data access, and which we adopt in the following.

In description logics (DLs) \cite{BCMNP03}, the domain of interest is modeled
by means of \emph{concepts}, representing classes of objects, and \emph{roles},
representing binary relations between objects\footnote{For simplicity of
 presentation, we do not distinguish here between roles, corresponding to OWL~2
 object properties, and attributes, corresponding to OWL~2 data properties (see
 \cite{CDLL*09}).  But all our results would carry over unchanged.}
In \dlliter, concepts $C$ and roles $U$
obey to the following syntax:
\[
  B ~::=~  N ~\mid~ \SOMET{U}
  \qquad\qquad
  C ~::=~  B ~\mid~ \SOME{U}{B}
  \qquad\qquad
  U ~::=~  P ~\mid~ \INV{P}
\]
$P$ denotes a \emph{role name}, and $\INV{P}$ an \emph{inverse role}, which
swaps the first and second components of $P$.  $N$ denotes a \emph{concept
 name}, and $B$ a \emph{basic concept}, which is either simply a concept name,
or the projection of a role $P$ on its first component ($\SOMET{P}$) or its
second component ($\SOMET{P^-}$).  In the concept $\SOME{U}{B}$, the projection
on the first (resp., second) component of $U$ can be further qualified by
requiring that the second (resp., first) component of $U$ is an instance of the
basic concept $B$.

In DLs, domain knowledge is divided into an intensional part, called
\emph{TBox}, and an extensional part, called \emph{ABox}.  Specifically, a
\dlliter \emph{ontology} is a pair $(\T,\A)$, where the TBox $\T$ is a finite
set of (concept and role) \emph{inclusion assertions} of the forms $B\ISA C$
and $U_1\ISA U_2$, and of \emph{disjointness assertions} of the form
$\disjoint{B_1}{B_2}$ and $\disjoint{U_1}{U_2}$.  Instead, the ABox $\A$ is a
finite set of \emph{facts} (also called membership assertions) of the forms
$N(c_1)$ and $P(c_1,c_2)$, where $N$ and $P$ occur in $\T$, and $c_1$ and $c_2$
are constants.

The semantics of a \dlliter ontology is given in terms of first-order
\emph{interpretations} $\I=(\dom,\Int{\cdot})$, where $\dom$ is the
interpretation domain and $\Int{\cdot}$ is an interpretation function that
assigns to each concept $C$ a subset $\Int{C}\subseteq\dom$ and to each role
$U$ a binary relation $\Int{U}\subseteq\dom\times\dom$, capturing the intuitive
meaning of the various constructs (see~\cite{CDLLR07} for details).
An interpretation that satisfies all assertions in $\T$ and $\A$ is called a
\emph{model} of the ontology $(\T,\A)$, and the ontology is said to be
\emph{satisfiable} if it admits at least one model.

\smallskip
\noindent
\textbf{Queries.}  As usual (cf.\ OWL~2~QL), answers to queries are formed by
terms denoting individuals explicitly mentioned in the ABox.  The \emph{domain
 of an ABox} $A$, denoted by $\adom{A}$, is the (finite) set of terms appearing
in $A$.
A \emph{union of conjunctive queries} (UCQ) $q$ over a KB $(T,A)$ is a FOL
formula of the form
$\bigvee_{1\leq i\leq n}\exists\vec{y_i}\per\conj_i(\vec{x},\vec{y_i})$ with free
variables $\vec{x}$ and existentially quantified variables
$\vec{y}_1,\ldots,\vec{y}_n$.  Each $\conj_i(\vec{x},\vec{y_i})$ in
$q$ is a conjunction of atoms of the form $N(z)$, $P(z,z')$, where $N$
and $P$ respectively denote a concept and a role name occurring in
$T$, and $z$, $z'$ are constants in $\adom{A}$ or variables in
$\vec{x}$ or $\vec{y_i}$, for some $i\in\{1,\ldots,n\}$.
%
The \emph{(certain) answers} to $q$ over $(T,A)$ is the set $\ans{q,T,A}$ of
substitutions 
$\sigma$ of the free variables of $q$ with constants in $\adom{A}$ such that
$q\sigma$ evaluates to true in every model of $(T,A)$.
If $q$ has no free variables, then it is called \emph{boolean} and its certain
answers are either $\true$ or $\false$.

We compose UCQs using ECQs, i.e., queries of the query language
\textit{EQL-Lite}(UCQ)~\cite{CDLLR07b}, which is the FOL query language whose
atoms are UCQs evaluated according to the certain answer semantics above. An
\emph{ECQ} over $T$ and $A$ is a possibly open formula of the form
\[
  Q ~:=~
  [q] ~\mid~
\lnot Q ~\mid~ Q_1\land Q_2 ~\mid~
  \exists x\per Q
\]
where $q$ is a UCQ.
%
The \emph{answer to $Q$ over $(\T,\A)$}, is the set $\Ans{Q,\T,\A}$ of
tuples of constants in $\adom{A}$ defined by composing the certain
answers $\ans{q,T,A}$ of UCQs $q$ through first-order constructs, and
interpreting existential variables as ranging over $\adom{A}$.

Finally, we recall that \dlliter enjoys the \emph{FO rewritability} property,
which states that for every UCQ $q$, $\ans{q,\T,\A} =
\ans{\textsc{rew}(q),\emptyset,\A}$, where $\textsc{rew}(q)$ is a UCQ computed by the
reformulation algorithm in \cite{CDLL*09}. Notice that this algorithm can be
extended to ECQs \cite{CDLLR07b}, and that its effect is to ``compile away''
the TBox. Similarly, ontology satisfiability is FO rewritable for \dlliter
TBoxes \cite{CDLL*09}, which states that for every TBox $\T$, there exists a
boolean first-order query $\qunsat{\T}$ such that for every non-empty ABox
$\A$, we have that $(\T,\A)$ is satisfiable iff
$\ans{\qunsat{\T},\T,\A}=\false$.


\smallskip
\noindent
\textbf{Ontology-Based Data Access (OBDA).}
\label{sec:OBDASystem}
In an OBDA system, a relational database is connected to an ontology
representing the domain of interest by a mapping, which relates database values
with values and (abstract) objects in the ontology (cf.\ \cite{CDLL*09}).
In particular, we make use of a countably infinite set  $\VALUE$ of
values and a set $\Lambda$ of function symbols, each with an
associated arity. We also define the set $\CONST$ of constants as the
union of  $\VALUE$ and the set
$\{f(d_1,\ldots,d_n) \mid f \in \Lambda \text{ and } d_1,\ldots,d_n \in
\VALUE \}$ of \emph{object terms}.

Formally, an OBDA system is a structure $\obdaSym = \obdaTup$, where:
\begin{inparaenum}[\it (i)]
\item $\DBschema =\set{R_1, \ldots, R_n}$ is a database schema, constituted by
  a finite set of relation schemas;
\item $\TBox$ is a \dlliter TBox;
\item $\mapping$ is a set of mapping assertions, each of the form
  $\map{\Phi(\vec{x})}{\Psi(\vec{y},\vec{t})}$,
  where:
  \begin{inparaenum}
  \item $\vec{x}$ is a non-empty set of variables,
  \item $\vec{y} \subseteq \vec{x}$,
  \item $\vec{t}$ is a set of object terms of the form $f(\vec{z})$, with $f
    \in \Lambda$ and $\vec{z} \subseteq \vec{x}$,
  \item $\Phi(\vec{x})$, which also called as \emph{source query}, is an arbitrary SQL query over $\DBschema$, with
    $\vec{x}$ as output variables, and
  \item $\Psi(\vec{y},\vec{t})$, which also called as \emph{target query}, is a CQ over $\TBox$ of arity $n >0$ without
    non-distinguished variables, whose atoms are over the variables $\vec{y}$
    and the object terms $\vec{t}$.
  \end{inparaenum}
\end{inparaenum}

Given a database instance $\DBinst$ made up of values in $\VALUE$ and
conforming to schema $\DBschema$, and given a mapping $\mapping$, the
\emph{virtual ABox} generated from $\DBinst$ by a mapping assertion
$m=\map{\Phi(x)}{\Psi(y,t)}$ in $\mapping$ is
$m(\DBinst)=\bigcup_{v\in\evaluate(\Phi,\DBinst)} \Psi[x/v]$, where
$\evaluate(\Phi,\DBinst)$ denotes the evaluation of the SQL query $\Phi$ over
$\DBinst$, and where we consider $\Psi[x/v]$ to be a set of atoms (as opposed
to a conjunction).  The ABox generated from $\DBinst$ by the mapping $\mapping$
is $\mapping(\DBinst)=\bigcup_{m \in \mapping} m(\DBinst)$.
Notice that $\adom{\mapping(\DBinst)} \subseteq \CONST$.
As for ABoxes, the active domain $\adom{\DBinst}$
of a database instance $\DBinst$ is the set of values occurring in
$\DBinst$. Notice that $\adom{\DBinst} \subseteq \VALUE$.
%
%
Given an OBDA system $\obdaSym = \obdaTup$ and a database instance
$\DBinst$ for $\DBschema$, a \emph{model} for $\obdaSym$ wrt $\DBinst$ is a
model of the ontology $(\TBox,\mapping(\DBinst))$. We  say that
$\obdaSym$ wrt $\DBinst$ is satisfiable if it admits a model wrt $\DBinst$.

A UCQ $q$ over an OBDA system $\obdaSym = \obdaTup$ and a relational
instance $\DBinst$ for $\DBschema$ is simply an UCQ over
$(\TBox,\mapping(\DBinst))$. To compute the certain answers of $q$
over $\obdaSym$ wrt $\DBinst$, we follow the standard three-step approach~\cite{CDLL*09}:
\begin{inparaenum}[\it (i)]
\item $q$ is \emph{rewritten} to compile away $\TBox$, obtaining $q_r =
  \rew{q,\TBox}$;
\item the mapping $\mapping$ is used to \emph{unfold} $q_r$ into a query over
  $\DBschema$, denoted by $\unfold{q_r}{\mapping}$, which turns out to be an
  SQL query \cite{PLCD*08};
\item such a query is executed over $\DBinst$, obtaining the certain answers.
\end{inparaenum}
For an ECQ, we can proceed in a similar way, applying the rewriting and
unfolding steps to the embedded UCQs.  It follows that computing certain
answers to UCQs/ECQs in an OBDA system is FO rewritable. Applying the
unfolding step to $\qunsat{\TBox}$, we obtain also that satisfiability in
$\obdaSym$ is FO rewritable.




\section{Semantically-enhanced Artifact Systems}
\label{sec:sas}

In this section we introduce Semantically-enhanced Artifact Systems ($\sgds$s),
which generalize the framework of Semantically-Governed Data-Aware Processes
(SGDAPs) presented in \cite{CDLMS12}.

Intuitively, \sgds models systems in which artifacts progress according to
their lifecycles, and in which the evolution of the entire system is understood
through the conceptual lens of an OWL~2~QL ontology. In accordance with the
literature \cite{DHPV09}, it is assumed that artifacts are equipped with a
relational information model. More specifically, a \sgds is constituted by:
\begin{inparaenum}[\it (i)]
\item A \emph{Relational Layer}, which account for the
  (relational) information models of the artifacts, and which employs a
  global transition relation to abstractly capture the step-by-step
  evolution of the system as a whole.
\item A \emph{Semantic Layer}, which contains an OWL~2~QL ontology
  that conceptually accounts for the domain under study.
\item A set of \emph{mapping assertions} describing how to
  virtually project data concretely maintained at the Relational Layer
  into concepts and relations modeled in the Semantic Layer, thus
  providing a link between the artifact information models and the ontology.
%
%
\end{inparaenum}
%
%
%

\begin{figure}[t!]
\centering
\includegraphics[width=\textwidth]{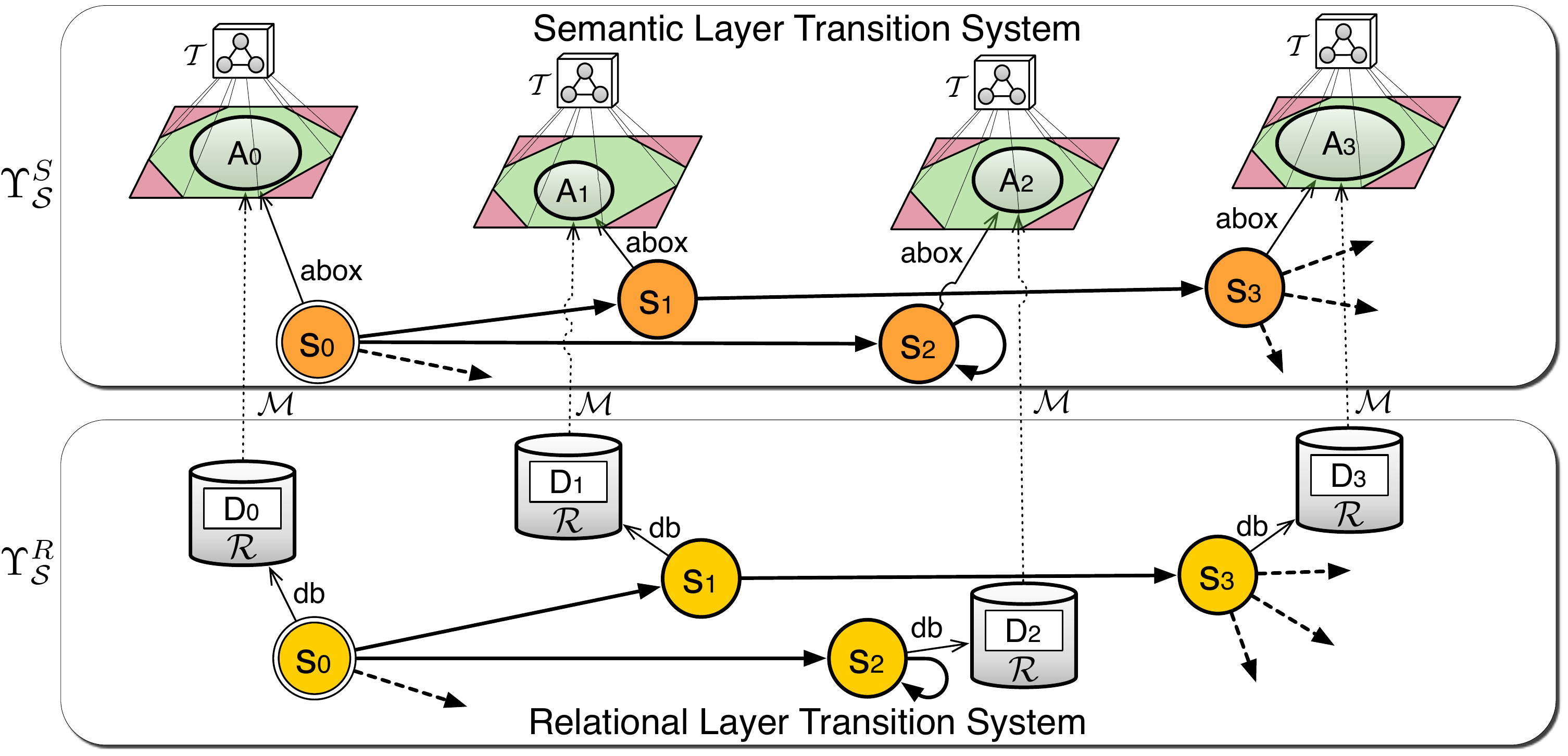}
\vspace*{-6mm}
\caption{\sgds illustration \label{fig:evolution}}
\end{figure}


In the following, we assume a countably infinite set of values $\VALUE$.
Formally, a \sgds $\sgdsSym$ is a tuple
$\sys = \tup{\DBschema,\DBinit,\F, \TBox, \mapping}$, where:
\begin{inparaenum}[\it (i)]
\item $\DBschema$ is a database schema that incorporates the schemas
  of all artifact information models present in the Relational Layer;
\item $\DBinit$ is a database instance made up of
  values in $\VALUE$ and conforming to $\DBschema$, which represents
  the initial state of the Relational Layer,;
\item $\F\subseteq \Instances\times \Instances$ is the transition
  relation that describes the overall progression mechanism of the
  Relational Layer, where $\Instances$ is the set of all instances
  made up of values in $\VALUE$ and conforming to $\DBschema$;
%
%
\item $\TBox$ is a $\dlliter$ TBox; 
\item $\mapping$ is a set of mapping assertions that connect $\DBschema$
  to $\TBox$, following the approach described in
  Section~\ref{sec:preliminaries}.
\end{inparaenum}
By observing that the triple $\tup{\DBschema,\TBox,\mapping}$
constitutes, in fact, an OBDA system, we notice that $\sys$ can be
considered as an OBDA system equipped with a transition relation that
accounts for the dynamics of the system at the level of $\DBschema$,
starting from the initial state $\DBinit$.


%

\subsection{Execution Semantics}
\label{sec:executionSemantic}

The execution semantics of a $\sgds$ \sys is provided by means of transition
systems. While the temporal structure of such transition systems is
fully determined by the transition relation of \sys, the content of
each state in the transition system depends on whether the dynamics of
$\sgds$s are understood directly at the Relational Layer, or through
the conceptual lens of the Semantic Layer ontology. In the former
case, each state is associated to a database instance that represents
the current snapshot of the artifact information models, whereas in
the latter case each state is associated to an ABox that represents
the current state of the system, as understood by the Semantic Layer. 

Following this approach,  the execution semantics of $\sgdsSym$ is captured in terms of two
transition systems, one describing the allowed evolutions at the
Relational Layer (\emph{Relational Transition System - RTS}),
and one abstracting them at the Semantic Layer (\emph{Semantic
  Transition System} - STS). Figure~\ref{fig:evolution} provides a
graphical intuition about the RTS and STS, and their
interrelations. 


\smallskip
\noindent
\textbf{RTS.} 
Given a \sgds $\sgdsSym=\tup{\DBschema,\DBinit,\F, \TBox, \mapping}$,
its RTS $\TSAsym_\sgdsSym$ is defined as a tuple $\TSAtup$, where:
\begin{inparaenum}[\it (i)]
\item  $\TSAs$ is a set of states,
\item $s_0 \in \TSAs$,
\item $\db$ is a function that, given a state in $\TSAs$, returns a
  corresponding database instance (conforming to $\DBschema$),
\item $\TSAt \subseteq \TSAs \times \TSAs$ is the transition relation.
\end{inparaenum} 
The components $\TSAs$, $\TSAt$ and $\db$ of $\TSAsym_\sgdsSym$ are
defined by
simultaneous induction as the smallest sets satisfying the following conditions:
\begin{compactitem}
\item $\db(s_0)  = \DBinit$;
\item for every databases instance $\DBinst'$ such that
  $\tup{\db(s),\DBinst'} \in \F$:
   \begin{compactitem}
   \item if there exists $s' \in \TSAs$ such that $\db(s') =
     \DBinst'$, then $s \TSAt s'$;
   \item otherwise, if $\obdaSym = \tup{\DBschema,\TBox,\mapping}$ is
     \emph{satisfiable} wrt $\DBinst'$, then $s' \in \TSAs$, $s \TSAt
     s'$ and $\db(s') = \DBinst'$, where $s'$ is a fresh state.
   \end{compactitem}
\end{compactitem}


\noindent
The satisfiability check done in the last step of the
RTS construction accounts for the semantic \emph{governance}
(cf.~Section \ref{sec:introduction}): a transition is preserved in the
RTS only if the target state does not violate any ontological
constraints of the Semantic Layer, otherwise it is rejected \cite{CDLMS12}.

\smallskip
\noindent
\textbf{STS.}  
Given a \sgds $\sgdsSym=\tup{\DBschema,\DBinit,\F, \TBox, \mapping}$,
its STS $\TSSsym_\sgdsSym$ is defined as a tuple $\TSAtup$, which is
similar to an RTS, except from the fact that states are attached to
ABoxes, not database instances.
 In particular, $\TSSsym_\sgdsSym$ is defined as a
``virtualization'' of the RTS $\TSAsym_\sgdsSym=\TSAtup$ at the
Semantic Layer: it maintains the structure of $\TSAsym_\sgdsSym$
unaltered, reflecting that the progression of the system is determined
at the Relational Layer, but it associates each state to a virtual
ABox obtained from the application of the mapping specification
$\mapping$ to the database instance associated by $\TSAsym_\sgdsSym$
to the same state.  Formally, the transition relation $\mapping$ is
equivalent to the one of the $\sgdsSym$, and the $\abox$ function of
$\TSSsym_\sgdsSym$ is defined as follows: for each $s \in \TSSs$, 
$\abox(s) = \mapping(\db(s))$.

\section{Verification of Semantically-Enhanced Artifact Systems}
\label{sec:verification}
Given a \sgds \sgdsSym, we are interested in studying verification of semantic
dynamic/temporal properties specified over the Semantic Layer, i.e., to be
checked against the STS $\TSSsym_\sgdsSym$. As verification formalism, we
consider a variant of first-order $\mu$-calculus \cite{Emer96b,Stir01},
called $\muladl$ \cite{BCMD*13,CKMSZ13}. We observe that $\mu$-calculus is one
of the most powerful temporal logics: it subsumes LTL, PSL, and CTL*
\cite{ClGP99}. The logic $\muladl$ extends propositional $\mu$-calculus by
allowing to query the states of the STS using the first-order epistemic queries
introduced in Section~\ref{sec:preliminaries}. In $\muladl$, first-order
quantification is restricted to objects that are present in the current ABox,
and can be used to relate objects across states.
The syntax of \muladl is as follows:
\[
  \Phi ::= Q ~\mid~ \lnot \Phi ~\mid~ \Phi_1 \land \Phi_2 ~\mid~ \exists
  x.\inadom(x) \land \Phi ~\mid~ \DIAM{\Phi} ~\mid~ Z ~\mid~ \mu Z.\Phi
\]
Where $Q$ is an ECQ over $\TBox$, $Z$ is a second-order variable denoting a
0-ary predicate, $\mu$ is the least fixpoint operator, and the special
predicate $\inadom(x)$ is used to indicate that $x$ belongs to the current
active domain, i.e., it is mentioned in some concept or role of the current
ABox. For a detailed semantics of \muladl, refer to
\cite{BCMD*13,CKMSZ13}.

Given a  \sgds \sgdsSym, we show that verification
of $\muladl$ properties over the STS $\TSSsym_\sgdsSym$ can be reduced
to verification of $\muladom$ \cite{CaDM13} properties over the RTS
$\TSAsym_\sgdsSym$, where $\muladom$ is a logic similar to $\muladl$,
except for the local formula $Q$, which is an (open) first-order query over the database
schema in the Relational Layer.

\begin{theorem}
\label{thm:verification}
For every \sgds \sgdsSym and \muladl property $\Phi$, there exists a
\muladom property $\Phi'$ such that $\TSAsym_\sgdsSym$ satisfies
$\Phi$ if and only if $\TSSsym_\sgdsSym$ satisfied $\Phi'$.
\end{theorem}

\begin{proof}[sketch]
The proof extends the one in \cite{CDLMS12}. We recap such a proof
here, and show how to extend it to the logic $\muladl$. 
The construction of $\Phi'$ is done by \emph{rewriting} and
\emph{unfolding} the $\Phi$, by suitably extending the
notions of rewriting and unfolding in OBDA to the case of temporal
properties \cite{PLCD*08}.

As for the rewriting step, the TBox $\TBox$ does not focus on the dynamics of the
system, but only on the information maintained by
$\sgdsSym$. Therefore, the rewriting of $\Phi$ wrt
the TBox $\TBox$ is done by separating the treatment of the
dynamic part, which is simply maintained unaltered, to the one of the
embedded ECQs, which is handled as in standard OBDA.

The unfolding step is managed in a similar way. Both $\TSSsym_\sgdsSym$ and $\TSAsym_\sgdsSym$
have the same structure, reflecting that the dynamics
of the SAS are determined at the Relational
Layer. Hence, the unfolding of $\rew{\Phi,\TBox}$ is handled by
maintaining the dynamic component unaltered, and unfolding all the embedded ECQs wrt the
mapping assertion $\mapping$. The correctness of this approach has
been proven in \cite{CDLMS12}, which however requires here to extend
the proof to the following two cases, related to first-order quantification:
\begin{inparaenum}[\it (i)]
\item the unfolding of $\inadom(x)$;
\item the unfolding of quantification across states, i.e., $\exists
  x.\inadom(x) \land \Phi$.
\end{inparaenum}
As for $\inadom(x)$, it can
be defined in terms of the following UCQ, which can then be unfolded
in the usual way:
{\footnotesize{
\[
\inadom(x) \equiv \bigvee_{N \text{ in } \TBox} N(x) \lor \bigvee_{P \text{ in } \TBox}
(P(x,\_) \lor P(\_,x))
\]
}}%
Quantification across states is dealt as in standard OBDA, i.e.:
{\small{
\begin{align*}
  \unfold{\exists x.\inadom(x) \land \Phi}{\mapping} =
  \textstyle\bigvee_{(f/n) \in \skolem(\mapping)} &\exists
  x_1,\ldots,x_n. 
  \unfold{\inadom(f(x_1,\ldots,x_n))}{\mapping} \\
  &\land \unfold{\Phi[x/f(x_1,\ldots,x_n)]}{\mapping}
\end{align*}
}}%
where $\skolem(\mapping)$ is the set of function symbols used to
construct object terms in $\mapping$. To conclude the proof, we need
to observe that the unfolded formula is in $\muladom$, i.e., it
preserves the property that quantification ranges over the current
active domain. This can be easily confirmed by noticing that a live
object term in a state of $\TSSsym_\sgdsSym$ can only be produced starting from
live values in the corresponding state of
$\TSAsym_\sgdsSym$. 


\end{proof}

\section{SAS Instantiation: The OBGSM Tool}
\label{sec:obgsmtool}

In this section, we show how the formal framework of $\sgds$s can be concretely instantiated, in the case where the transition relation at the Relational Layer is obtained from artifacts specified using the Guard-Stage-Milestone (GSM) approach. In particular, we show how existing techniques and tools have been suitably combined into a tool, called  \emph{\systemName}, which enables the verification of GSM-based artifacts with a Semantic Layer.

%

\smallskip
\noindent
\textbf{GSM Overview.}
For the sake of space, here we provide a general overview of the GSM
methodology and we refer to \cite{HDDF*11,DaHV13} for more detailed and formal
definitions.  GSM is a declarative modeling framework that has been designed to
simplify design and maintenance of business models.  The GSM information model
uses (possibly nested) attribute/value pairs to capture the domain of interest.
The key elements of a lifecycle model are \emph{stages}, \emph{milestones} and
\emph{guards} (see Figure~\ref{fig:claimingstage}).  Stages (represented as
rounded rectangles) are possibly hierarchical clusters of activities, intended
to update and extend the data of the information model. They are associated to
milestones (circles), business operational objectives which can be achieved
while the stage is under execution. Each stage has one or more guards
(diamonds), which control the activation of stages and, like milestones, are
described in terms of data-aware expressions, involving conditions over the
artifact information model.

\smallskip
\noindent
\textbf{The Ontology-Based GSM (OBGSM) Tool. } \xspace
%
%
%
\systemName exploits two already existing tools to provide its functionalities:
\begin{inparaenum}[\it (i)]
\item \ontop\footnote{\url{http://ontop.inf.unibz.it/}}, a JAVA-based
  framework for OBDA.
%
\item the \emph{GSMC model checker}, developed within the EU FP7 Project
  ACSI\footnote{``Artifact-Centric Service Interoperation'', see
   \url{http://www.acsi-project.eu/}} to verify GSM-based artifact-centric
  systems against temporal/dynamic properties \cite{BeLP12b}.
%
\end{inparaenum}
An important observation is related to semantic governance in this
setting: since the construction of the RTS for GSM is handled
internally by GSMC, it is not possible (at least for the time being)
to prune it so as to remove inconsistent states. Therefore, in the
following we assume that all the states in the RTS are consistent with
the constraints of the Semantic Layer. This can be trivially achieved
by, e.g., avoiding to use negative inclusion assertions in the TBox,
which are the only source of inconsistency for OWL~2~QL. A more
elaborated discussion on this topic is included in Section~\ref{sec:discussion}.

The main purpose of \systemName is: given a temporal property specified over
the Semantic Layer of the system, together with mapping assertions whose
language is suitably shaped to work with GSMC, automatically rewrite and
unfold the property by producing a corresponding translation that can be directly
processed by the GSMC model checker.  This cannot be done by solely relying on
the functionalities provided by \ontop, for two reasons:
\begin{inparaenum}[\it (i)]
\item \systemName deals with temporal properties specified in a fragment of
  $\muladl$, and not just (local) ECQs;
\item the mapping assertions are shaped so as to reflect the specific query
  language supported by GSMC, guaranteeing that the rewriting and unfolding
  process produces a temporal property expressed in the input language of GSMC.
\end{inparaenum}
In particular, we note that GSMC is not able to process the entire
$\muladom$ logic, but only its CTL fragment, denoted by \ctla. This
requires also to restrict the $\muladl$ verification formalism
accordingly, in particular focusing on its CTL fragment, denoted by \ctladl.

As depicted in the Figure~\ref{fig:SystemArchitecture},
the workflow of \systemName is as follows:
\begin{inparaenum}[\it (i)]
\item The tool reads and parses the input
conceptual temporal property $\Phi$, the input ontology, and the input
mapping declaration $\mathcal{M}$.
\item It then \emph{rewrites} the input conceptual temporal property
  $\Phi$ based on the input ontology (TBox) in order to compile away
  the TBox.
%
%
  This step produces a \emph{rewritten temporal property}
  $\textsc{rew}(\Phi,\TBox)$.
\item The rewritten property  $\textsc{rew}(\Phi,\TBox)$ is \emph{unfolded} by
  exploiting $\M$.  The final temporal property $\Phi_{GSM} =
  \textsc{unfold}(\textsc{rew}(\Phi,\TBox), \mathcal{M})$
  obeys to the syntax expected by GSMC, and is such that verifying $\Phi$
  over the transition system of the GSM model under study
  after projecting its states into the Semantic Layer through $\M$,
  is equivalent to verifying $\Phi_{GSM}$ directly over the GSM model
  (without considering the Semantic Layer).
\item GSMC is invoked by passing $\Phi_{GSM}$ and the
  specification file of the GSM model under study.
\end{inparaenum}

Notice that the correctness of the translation is guaranteed by the fact that \systemName manipulates the local components of the query $\Phi$
according to the standard rewriting and unfolding algorithms, while
maintaining untouched the temporal structure of the property. This has
been proven to be the correct way of manipulating the property
(cf.~Theorem~\ref{thm:verification}). Furthermore, since the temporal
component of the formula is left untouched, it is guaranteed that
$\ctladl$ properties are rewritten and unfolded into corresponding
$\ctla$ properties.
%


\begin{figure}[tbp]
  \centering
  \hspace*{-3mm}\includegraphics[width=1.05\textwidth]{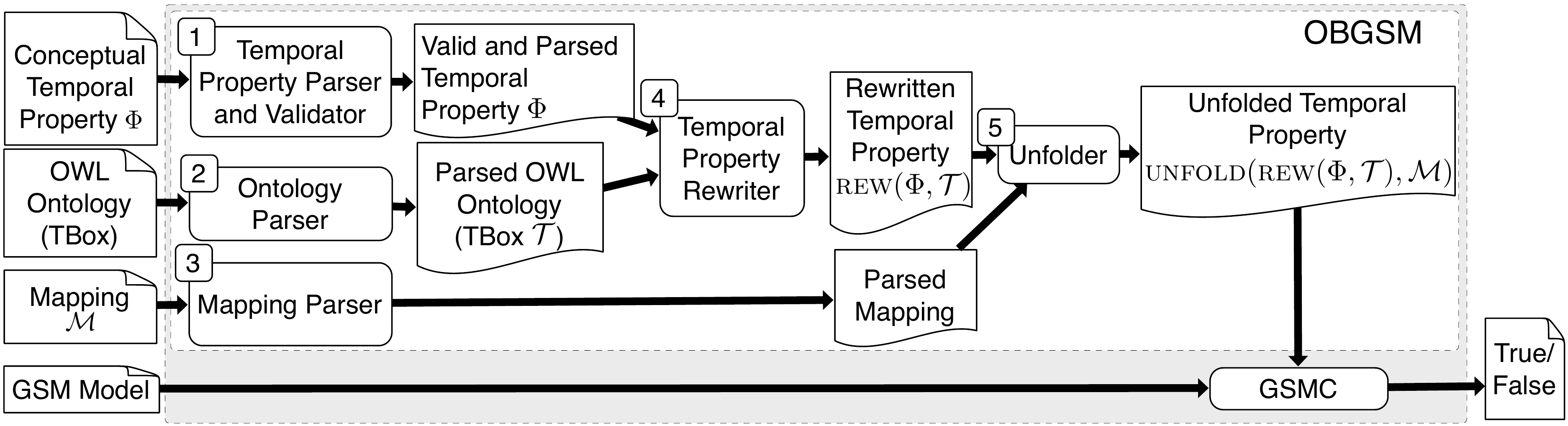}
  \caption{\systemName System Architecture \label{fig:SystemArchitecture}}
\end{figure}


The verification language of \systemName relies on CTL, in accordance to the
input verification language of GSMC \cite{BeLP12b,ACSI-D2.2.3} with some
adjustment. Since we are querying over an OWL ontology, we rely for the local
query language on SPARQL, in accordance to the query language supported by
\ontop. More specifically, the syntax is as follows:

\vspace{-3mm}
{\footnotesize{
\begin{equation*}
\begin{split}
  \formula ~::=&~ [~\query~]
~\mid~ ( \formula )
~\mid~ \andTemp{\formula_1}{\formula_2}
~\mid~ \orTemp{\formula_1}{\formula_2}
~\mid~ \implTemp{\formula_1}{\formula_2}
~\mid~ \notTemp~\formula
~\mid~ \AG~\formula
~\mid~ \EG~\formula \\
~\mid&~ \AF~\formula
~\mid~ \EF~\formula
~\mid~ \AX~\formula
~\mid~ \EX~\formula
~\mid~ \Auntil{\formula_1}{\formula_2}
~\mid~ \Euntil{\formula_1}{\formula_2}  \\
~\mid&~ \forallTemp\ \variables\ .\ \forallQuantification
~\mid~ \existsTemp\ \variables\ .\ \existsQuantification \\
\end{split}
\end{equation*}
\begin{equation*}
\begin{split}
\forallQuantification ~::=&~ \implTemp{[~\query~]}{\formula}
~\mid~ \forallTemp\ \variables\ .\ \forallQuantification
~\mid~ [~\query~] \\
\existsQuantification ~::=&~ \andTemp{[~\query~]}{\formula}
~\mid~ \existsTemp\ \variables\ .\ \existsQuantification
~\mid~ [~\query~] \\
\end{split}
\end{equation*}
}}%
where ``$\notTemp$'' denotes the negation, and $[\query]$ is
a 
\sparql \footnote{\url{http://www.w3.org/TR/sparql11-query/}} \sparqlSELECT
query, which only allows for:
\begin{inparaenum}[\it (i)]
\item prefixes declarations,
\item variables selections, and
\item a \sparqlWHERE clause that only contains triple patterns and some filters
  on variable/constant comparison.
\end{inparaenum}

The semantics of the temporal operators is the one of CTL \cite{ClGP99}.
Additionally, we also impose the following restrictions on first-order
quantification:
\begin{inparaenum}[\it (i)]
\item Only closed temporal formulae are supported for verification.
\item Each first-order quantifier must be ``guarded'' in such a way that it
  ranges over individuals present in the current active domain.  This active
  domain quantification is in line with GSMC, and also with the \muladl
  logic. As attested by the grammar above, this is syntactically guaranteed by
  requiring quantified variables to appear in $[\query]$ according to the
  following guidelines:
  \begin{inparaenum}[\it (a)]
  \item $\forall \vec{x}. \query(\vec{x}) \limp \phi$
  \item $\exists \vec{x}. \query(\vec{x}) \land \phi$.
  \end{inparaenum}
%
%
\item Quantified variables must obey to specific restrictions as
  follows: for each variable $y$ ranging over values, there must be at
  least one variable $x$ that ranges over object terms and that
  appears in the first component of the corresponding attribute (i.e.,
  $Attr(x,y)$ is present in the query, with $Attr$ being an attribute
  of the ontology), such that $x$ is quantified ``before'' $y$. For
  example, $\forall x. C(x) \implies \exists y. Attr(x.y)$ satisfies
  this condition, whereas $\exists y \exists x. Attr(x,y)$ does not.
\end{inparaenum}
%
%
%
%
These restrictions have been introduced so as to guarantee that the
conceptual temporal property can be translated into a corresponding
GSMC temporal property. In fact, GSMC poses several restrictions on
the way values (i.e., attributes of artifacts) can be accessed.


The structure of the \systemName mapping language is mainly borrowed from
\ontop{}.
%
%
%
%
Each mapping assertion is described by three components:
\begin{inparaenum}[\it (i)]
\item {\tt \footnotesize mappingId}, which provides a unique identifier for the
  mapping assertion.
\item {\tt \footnotesize target}, which contains the \emph{target
    query} (query over ontology) with the syntax adopted from
  \marianoSystem.
\item {\tt \footnotesize source}, which describes the \emph{source
    query}. For the grammar of the source query we used a modified
  grammar of the GSMC input language \cite{ACSI-D2.2.3}, with
  extensions that allow to ``extract'' artifact identifiers and their
  value attributes, so as to link them to the ontology. For more
  detail about \systemName, please refer to \cite{ACSI-D2.4.2}.
%
%
\end{inparaenum}

\section{Case Study: Energy Scenario}
\label{sec:caseStudy}


As a case study to demonstrate our approach we refer to the fragment of the
Energy Use Case Scenario developed within the ACSI Project \cite{ACSI-D5.5}.
We show how the Semantic Layer can be exploited in order to facilitate the
specification of temporal properties of interest, and discuss how these are
automatically translated into properties that can be directly verified by GSMC
over the Energy GSM model.

\smallskip
\noindent
\textbf{Energy Scenario at a Glance.}
%
We sketch the main aspects of the Energy Use Case, and refer
to \cite{ACSI-D5.5} for further details.
%
This use case focuses on the electricity supply exchange process between
electric companies inside a distribution network.
The
exchange
occurs at \emph{control points} (CP). Within a CP, a measurement of electricity
supply exchange takes place in order to calculate the fair remuneration that
the participating companies in the CP should receive. The measurement is done
by a \emph{meter reader company}, which corresponds to one of the companies
pertaining to that particular CP. The measurement results from the CP are then
submitted to the \emph{system operator}, who is in charge of processing the
results and publishing a \emph{CP monthly report}.
A participating company can raise an objection concerning the published
measurement.  Once all the risen objections are resolved, the report is
closed. The collection of CP monthly reports is then represented as a \emph{CP
 monthly report accumulator}.

In order to implement the sketched scenario using the GSM methodology, we model
the \emph{Control Point Monthly Report $($CPMR$)$} as an artifact which
contains the information about hourly measurements done in a particular CP
within a certain month.  The lifecycle of an instance of this artifact begins
once the hourly measurements are provided, and runs until the liquidation for
the CP measurements has started. The artifact lifecycle consists of three root
stages:
\begin{inparaenum}[\it (i)]
\item \emph{CPMRInitialization}, which is activated when a new instance of the
  CPMR artifact is created.
\item \emph{Claiming}, which handles the submission of measurements, and the
  subsequent reviewing procedure during which objections may be raised.
\item \emph{MeasurementUpdating}, which is activated once an event is raised,
  requesting for updating the measurement results.
\end{inparaenum}

Figure~\ref{fig:GSM Model} shows a GSM model of the \emph{Claiming} stage, on
which we focus here.  It contains the following five sub-stages:
\begin{inparaenum}[\it (i)]
\item \emph{Drafting}, which models the process of submitting and processing
  the CP hourly measurements and processing a draft of the CPMR.
\item \emph{Evaluating}, which is responsible for official measurement
  publication.
\item \emph{Reviewing}, which models the process of either accepting the
  official measurement, or deciding whether the difference is acceptable or
  not. Such decision may result in an objection either by the Metering Office
  (MO) or the Energy Control Department (ECD).
\item \emph{CreateObjection}, which corresponds to creation of a new Objection
  instance.
\item \emph{Closing}, which is responsible for closing the objection
  period. Once this period finishes and there are no requested objections, the
  monthly report is finalized and the \emph{CPMRFinished} milestone is
  achieved.
\end{inparaenum}
The collection of CP monthly reports for a specific month is then modeled by
the \emph{CPMR Monthly Accumulator $($CPMRMA$)$} artifact.

\begin{figure}[btp]
  \label{fig:claimingstage}
  \centering
  \includegraphics[width=\textwidth]{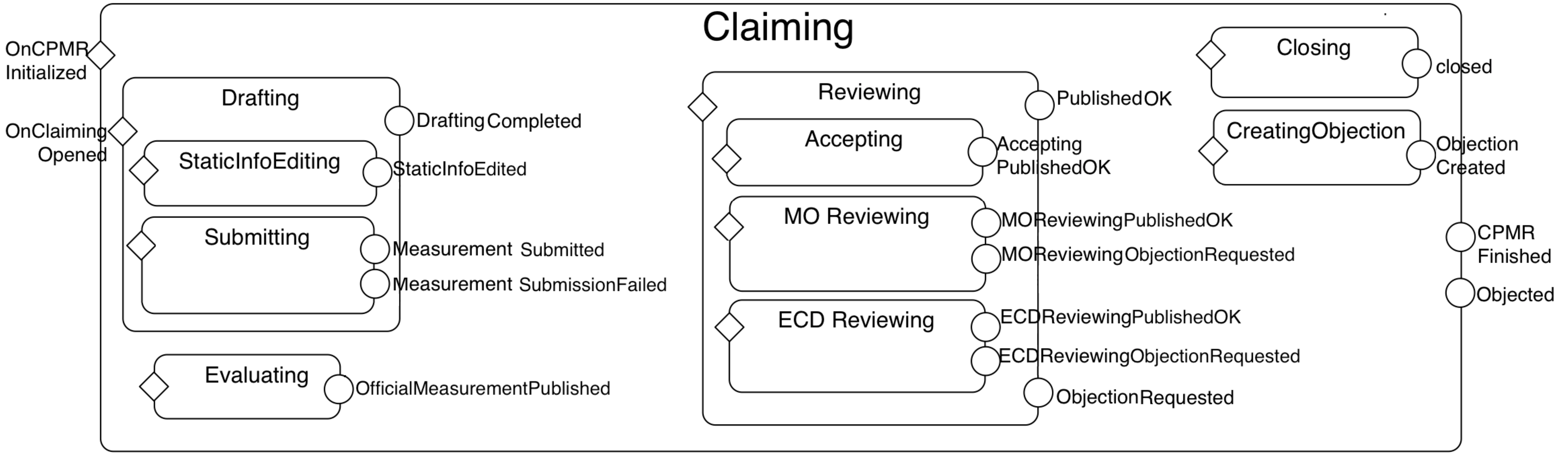}
  \caption{The Claiming Stage in the GSM Model for the CPMR Artifact}
  \label{fig:GSM Model}
\end{figure}

\smallskip
\noindent
\textbf{The Semantic Layer for the Energy Use Case.}
Here we provide a Semantic Layer built on top of the GSM model for the Energy
scenario, restricting our attention to the Control Point
Measurement Report (CPMR).

\smallskip
\noindent
\textbf{The Ontology TBox.}  We consider the following states of a control
point measurement report:
\begin{inparaenum}[\it (i)]
\item a finished CPMR (when the milestone \emph{CPMRFinished} is achieved),
\item a reviewed CPMR (after finishing the review inside the \emph{Reviewing}
  stage)
\item an accepted CPMR (when the milestone \emph{PublishedOK} is achieved)
\item an objected CPMR.
\end{inparaenum}
Following this intuition, a portion of the ontology TBox of the Semantic Layer
is depicted in Figure~\ref{fig:EnergyOntology}, together with its
representation as a UML Class Diagram.
%

\begin{figure}[btp]
  \raisebox{-0.6cm}{\includegraphics[width=0.6\textwidth]{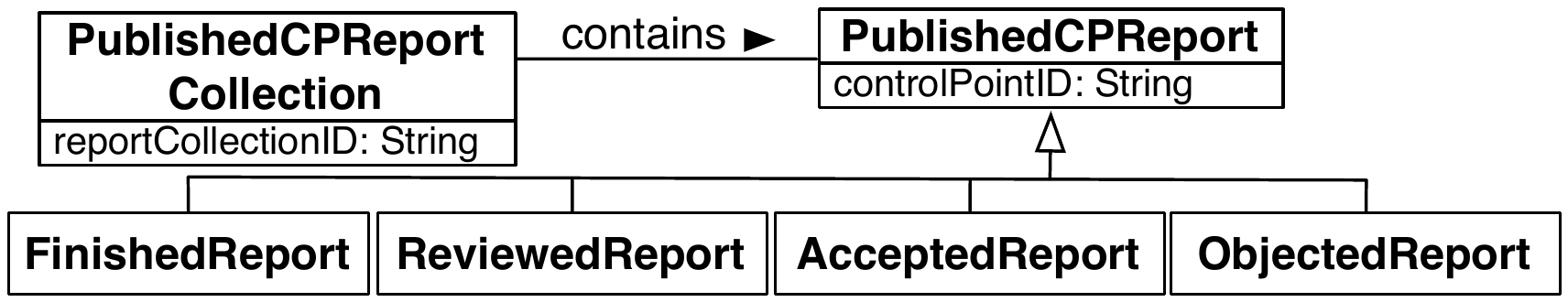}}
  \makebox[1cm][l]{\scriptsize
   $
   \begin{array}{rl}
     \SOMET{\mathit{contains}} &\ISA \mathit{PublishedCPReportColl.}\\
     \SOMET{\INV{\mathit{contains}}} &\ISA \mathit{PublishedCPReport} \\
     \SOMET{\mathit{controlPointID}} &\ISA \mathit{PublishedCPReport}\\
     \SOMET{\INV{\mathit{controlPointID}}} &\ISA \mathit{String}\\
     \mathit{FinishedReport} &\ISA \mathit{PublishedCPReport}\\
     &\cdots\\
   \end{array}
   $}

  \caption{Ontology for the CPMR reviewing process in the Energy scenario}
  \label{fig:EnergyOntology}
  \vspace{-5mm}
\end{figure}

\smallskip
\noindent
\textbf{The Mapping Assertions.}
%
%
We use mapping assertions as described in Section~\ref{sec:obgsmtool} in order
to link the GSM information model to the Semantic Layer. Due to the space limitations, we
only give an example of such a mapping in the \systemName mapping specification
language:

{\scriptsize
\begin{verbatim}
 mappingId  ReviewedReportMapping
 target     <"&:;cpmr/{$x}/"> rdf:type :ReviewedReport .
 source     get('x','CPMR')(GSM.isMilestoneAchieved('x','AcceptingPublishedOK') OR
                       GSM.isMilestoneAchieved('x','MOReviewingPublishedOK') OR
                       GSM.isMilestoneAchieved('x','ECDReviewingPublishedOK'))
\end{verbatim}
}%
\noindent
%
Intuitively, this mapping populates the $\mathit{ReviewedReport}$ concept with
the CPMR artifact instances, given the achievement of either of the milestones:
\emph{AcceptingPublishedOK}, \emph{MOReviewingPublishedOK}, or
\emph{ECDReviewingPublishedOK}. In the Semantic Layer,
$\mathit{ReviewedReport}$ intuitively represents a CPMR that has been
reviewed. In the Artifact Layer, this corresponds to the situation in which the
CPMR has been reviewed and accepted either by the electric company, or by the
metering office (MO), or by the Electric Control Department (ECD). This example
shows how such details can be abstracted away in the Semantic Layer, which does
not show that a $\mathit{ReviewedReport}$ is obtained by a (possibly complex)
chaining of achieved milestones in the underlying GSM model.




\smallskip
\noindent
\textbf{Verification within the Energy Use Case.}
We demonstrate now how the presence of the Semantic Layer may help when
specifying different temporal properties of interest.  For example, consider a
property saying that: \emph{``All control point monthly reports will eventually
 be finished''}.  This can be expressed as a conceptual temporal property over
the Semantic Layer defined above:
\[
  \AG (\forallTemp\ x.\ ([\mathit{PublishedCPReport}(x)] \implTempB \EF
  [\mathit{FinishedReport}(x)]) )
\]
Then it can be encoded in the machine-readable SPARQL-based verification
language, defined in Section \ref{sec:obgsmtool}, as follows:

%
{\fontsize{7.7}{7.7}\selectfont
\begin{verbatim}
AG (FORALL ?x.([PREFIX : <http://acsi/example/ACSIEnergy/ACSIEnergy.owl#>
                SELECT ?x WHERE { ?x a :PublishedCPReport}] ->
            EF [PREFIX : <http://acsi/example/ACSIEnergy/ACSIEnergy.owl#>
                SELECT ?x WHERE { ?x a :FinishedReport}]
          ));
\end{verbatim}
}

We now show how this kind of high-level properties are processed by \systemName
into temporal properties formulated over the GSM model and which can be fed to
the GSMC model checker.  Notice that without the Semantic Layer and
\systemName, the user would have to construct the low-level properties
manually.

We first perform a rewriting step, which expands concepts contained in the
property of interest using the ontology TBox. In our example, this results in
the conceptual property
$\AG (\text{PCR}(x) \implTempB \EF [\mathit{FinishedReport}(x)]))$,
in which:
\[
  \text{PCR}(x) =
  \begin{array}[t]{@{}l}
    \existsTemp\ y\ .\
    [\mathit{hasControlPointID}(x,y)] \orTempB
    [\mathit{ObjectedReport}(x)] \orTempB \\{}
    [\mathit{AcceptedReport}(x)] \orTempB
    [\mathit{PublishedCPReport}(x)] \orTempB
    [\mathit{ReviewedReport}(x)] \\{}
    \orTempB [\mathit{FinishedReport}(x)] \orTempB
    (\existsTemp\ z\ .\ [\mathit{contains}(z,x)])
  \end{array}
\]
In particular, the rewriting step employs reasoning over the \dlliter TBox by
taking into account the following knowledge contained in the Semantic Layer:
\begin{inparaenum}[\it (i)]
\item Those objects that have a control point ID (i.e., are in the domain of the
  attribute $\mathit{hasControlPointID}$), are instances of
  $\mathit{PublishedCPReport}$;
\item $\mathit{ObjectedReport}$ is a $\mathit{PublishedCPReport}$;
\item $\mathit{AcceptedReport}$ is a $\mathit{PublishedCPReport}$;
\item $\mathit{ReviewedReport}$ is a $\mathit{PublishedCPReport}$;
\item Those objects that are in the range of the role $contains$ are
  instances of $\mathit{PublishedCPReport}$;
\item $\mathit{FinishedReport}$ is a $\mathit{PublishedCPReport}$.
\end{inparaenum}

\begin{figure}[tp]
  \centering
  \footnotesize
\[
  \AG\ ( \forallGSM ('x',
  \begin{array}[t]{@{}l}
    \ '\mathit{CPMR}')(!(\GSMisMSachieved('x','\mathit{Objected}') \orTempB\\
    \GSMisMSachieved('x','\mathit{ObjectionRequested}')  \orTempB\\
    \GSMisMSachieved('x','\mathit{ObjectionCreated}') \orTempB\\
    \GSMisMSachieved('x','\mathit{MOReviewingObjectionRequested}') \orTempB\\
    \GSMisMSachieved('x','\mathit{ECDReviewingObjectionRequested}')   \orTempB\\
    \GSMisMSachieved('x','\mathit{PublishedOK}') \orTempB\\
    \GSMisMSachieved('x','\mathit{AcceptingPublishedOK}') \orTempB\\
    \GSMisMSachieved('x','\mathit{MOReviewingPublishedOK}') \orTempB\\
    \GSMisMSachieved('x','\mathit{ECDReviewingPublishedOK}') \orTempB\\
    \existsGSM( 'y', '\mathit{CPMRMA}' )  ( y./\mathit{CPMRMA}/\mathit{CPMRDATA} \implTempB \\
    \hspace{2cm}\textbf{exists}(\mathit{CPMRID} == {x./\mathit{CPMR}/\mathit{ID}}) )  \orTempB\\
    \GSMisMSachieved('x','\mathit{CPMRFinished}')) \orTempB\\
    EF  ( \GSMisMSachieved('x','\mathit{CPMRFinished}'))))\\
  \end{array}
\]
  \vspace{-5mm}
  \caption{Resulting temporal property over the GSM model}
  \label{fig:example-property}
  \vspace{-3mm}
\end{figure}


We then perform the unfolding step, which exploits the mapping assertions in
order to reformulate the rewritten property to the temporal property over the
GSM model.
Such property, shown in Figure~\ref{fig:example-property}, can be now be fed to
the GSMC model checker in order to validate the initial property claiming that
\emph{``all control point monthly reports will eventually be finished''}.
Given the significant difference in complexity between the property formulation
over the Semantic Layer and its corresponding translation over GSM, the
advantage of the presence of the Semantic Layer becomes apparent even for this
simple case study.  Conceptual temporal property formulation allows to hide
low-level details from the user and enables the modeler to rather focus on the
domain under study using the vocabulary he/she is familiar with.


\section{Discussion and Conclusion}
\label{sec:discussion}


In this paper we have introduced a framework in which an artifact system based
on the GSM model is connected through declarative mappings to a Semantic Layer
expressed as an OWL~2~QL ontology, and temporal properties expressed over the
Semantic Layer are reformulated in terms of the underlying GSM model and
verified over its execution, using state of the art OBDA and model checking
technology.
We have assumed in our development
that all states generated for the purpose of verification by the GSMC model
checker (which we use as a black box), are consistent with the Semantic Layer.
Indeed, inconsistent states could not be easily pruned away, as e.g., the
approach proposed in \cite{CDLMS12} would require.
Notice that in our case study presented in Section~\ref{sec:caseStudy} such
assumption holds since the TBox does not contain any disjointness assertions.

However, in general, the consistency of the states at the artifact level with
the TBox in the Semantic Layer is not guaranteed.  In such a situation the
strategy proposed here, based on delegating the verification of \ctla
properties to the existing model checker, cannot be followed directly.  One
possible solution to this is to provide a mechanism for the model checker to
exploit the ontology at the Semantic Layer to distinguish consistent from
inconsistent states of the low-level system.  The other possible solution is to
consider less expressive temporal logics, i.e., fragments of \ctla, and
investigate whether the check for consistency can be embedded in the formula to
verify.  These latter scenarios provide us with interesting problems for future
investigation.





\bibliographystyle{splncs03}
\bibliography{main-bib}

\end{document}